%% file: multi-slot.tex
\documentclass{sig-alternate-05-2015}

\RequirePackage[colorlinks,citecolor=blue,urlcolor=blue]{hyperref}
\usepackage{algorithm,booktabs, multirow, algpseudocode}
\RequirePackage{amsmath}
\RequirePackage[numbers]{natbib}

\usepackage{amsfonts,amssymb,graphicx,nicefrac,mathtools,subfigure}
\usepackage{enumerate, accents, dsfont}
\usepackage{paralist}
\input{Definitions} 

\newcommand{\bsx}{{\boldsymbol{x}}}

\newcommand{\bsa}{{\boldsymbol{a}}}
\newcommand{\dolb}{{\boldsymbol{\$}}}
\newcommand{\cs}{{\cal S}}
\newcommand{\cu}{{\cal U}}
\newcommand{\cv}{{\cal V}}
\newcommand{\cc}{{\cal C}}
\newcommand{\cp}{{\cal P}}
\newcommand{\ct}{{\cal T}}
\newcommand{\real}{{\mathbb R}}

\begin{document}
%
\doi{}

\isbn{}



%

\title{Constrained Multi-Slot Optimization for Ranking Recommendations}
%
%
%
%
%

\numberofauthors{3} 
%
\author{
%
%
\alignauthor
Kinjal Basu\\
       \affaddr{Department of Statistics}\\
       \affaddr{Stanford University}\\
       \affaddr{Stanford, CA USA}\\
       \email{kinjal@stanford.edu}
\alignauthor
Shaunak Chatterjee \\
       \affaddr{LinkedIn Corporation}\\
       \affaddr{Mountain View, CA}\\
       \email{shchatterjee@linkedin.com}
\alignauthor Ankan Saha\\
       \affaddr{LinkedIn Corporation}\\
       \affaddr{Mountain View, CA}\\
       \email{asaha@linkedin.com}
}

\maketitle
\begin{abstract}
\input{Abstract}    
\end{abstract}

%
%

%
%

%
%


\section{Introduction}
\label{sec:intro}
\input{Introduction}



\section{Feed Ranking as a MOO Problem}
\label{sec:problem_def}
\input{Problem_Defn}

\section{An Efficient Solution}
\label{sec:soln_prob}
\input{Problem_Soln}

\section{Modeling Interaction}
\label{sec:modeling_interaction}
\input{Modeling_Interaction}

\section{Solution to the QCQP}
\label{sec:sol_qcqp}
\input{Soln_QCQP}

\section{Experimental Results}
\label{sec:expt_results}
\input{Expt_results}

\section{Discussion}
\label{sec:disc}
\input{Discussion}

\section*{Acknowledgement}
We were sincerely like to thank Prof. Art Owen for the useful discussions and his help in providing some $(t,m,s)$-nets for our experiments.

\bibliographystyle{abbrv}
\bibliography{qp-multislot}

 
\end{document}

%% file: Abstract.tex
Ranking items to be recommended to users is one of the main problems
in large scale social media applications. This problem can be set up
as a multi-objective optimization problem to allow for trading off
multiple, potentially conflicting objectives (that are driven by those
items) against each other. Most previous approaches to this problem
optimize for a single slot without considering the interaction effect
of these items on one another.

In this paper, we develop a constrained multi-slot optimization
formulation, which allows for modeling interactions among the items on
the different slots. We characterize the solution in terms of problem
parameters and identify conditions under which an efficient solution
is possible. The problem formulation results in a quadratically
constrained quadratic program (QCQP). We provide an algorithm that
gives us an efficient solution by relaxing the constraints of the QCQP
minimally. Through simulated experiments, we show the benefits of
modeling interactions in a multi-slot ranking context, and the speed
and accuracy of our QCQP approximate solver against other state of the
art methods.

%% file: Introduction.tex










Ranking of items on a recommendation platform has become one of the
most important problems in most internet and social media
applications. Popular examples include the feed on most social media
applications like Facebook, LinkedIn and Instagram (Figure \ref{fig:feeds}). The People You May
Know (PYMK) application on LinkedIn which recommends professional
connections to the user, email digests with various news articles that
can be interesting to the user are other such examples (Figure \ref{fig:apps}). All of these
involve showing a list of ranked items from a larger set of candidates
to the user.  As the scale of these applications becomes progressively
larger with time, they have more data to optimize and provide a better
user experience. In the case of feed ranking at LinkedIn \cite{agarwal2014activity, agarwal2015personalizing} for example,
there can be multiple items of different types (articles shared by
first degree connections, status updates, connection updates, job
anniversaries to name a few) all of which need to be consolidated to
generate a ranked list which should be the most engaging to the
user. Due to the cost of optimization of most ranking objectives in
machine learning, scaling ranking methods to very large data sets
becomes difficult \cite{burges2011learning,valizadegan2009learning}. In most real life applications, these ranked
recommendations are generated by separately sorting each of these
classes of items (via predicted estimates of their click-through rate
(CTR) or some appropriate notion of engagement) and then mixing them
together using an algorithm.

At the same time the businesses operating these applications also aim
to grow by potentially trading off immediate user engagement (via CTR)
with monetization metrics like revenue from native ads and/or user
retention. This necessitates a formulation where multiple such
objectives can be traded off efficiently while trying to provide a
balanced user experience. There has been a plethora of research
addressing this problem, commonly known as Multi-Objective
Optimization (MOO)
\cite{agarwal2015constrained, agarwal2012personalized, basu_user, deb2014multi, konak2006multi, marler2004survey}, 
both in theory and practice.
\begin{figure*}
\centering
\begin{minipage}[b]{0.48\textwidth}
\begin{subfigure}[Linkedin]{
  \includegraphics[scale=0.062]{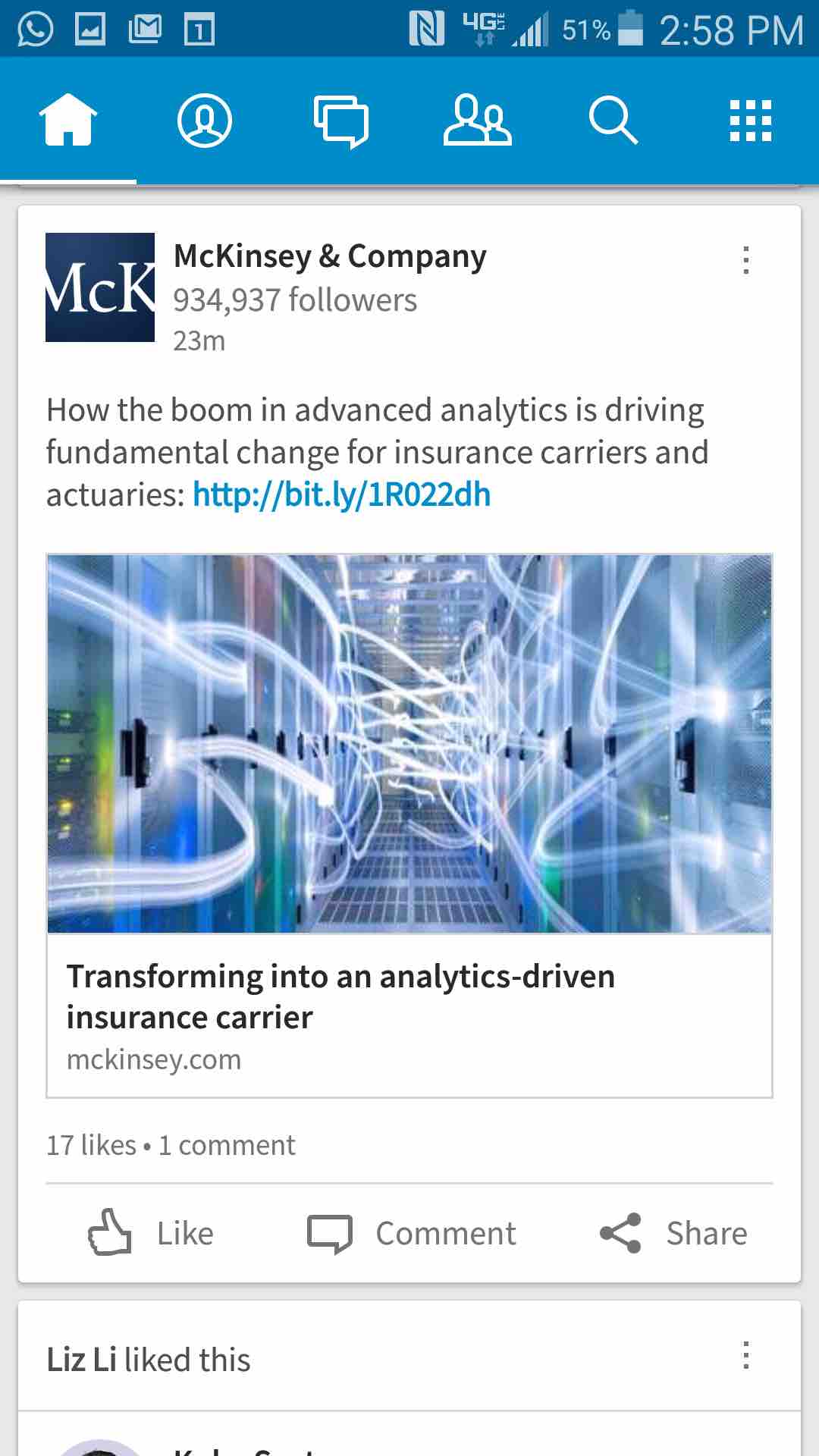}
  \label{fig:linkedin}}
\end{subfigure}
\begin{subfigure}[Facebook]{
  \includegraphics[scale=0.062]{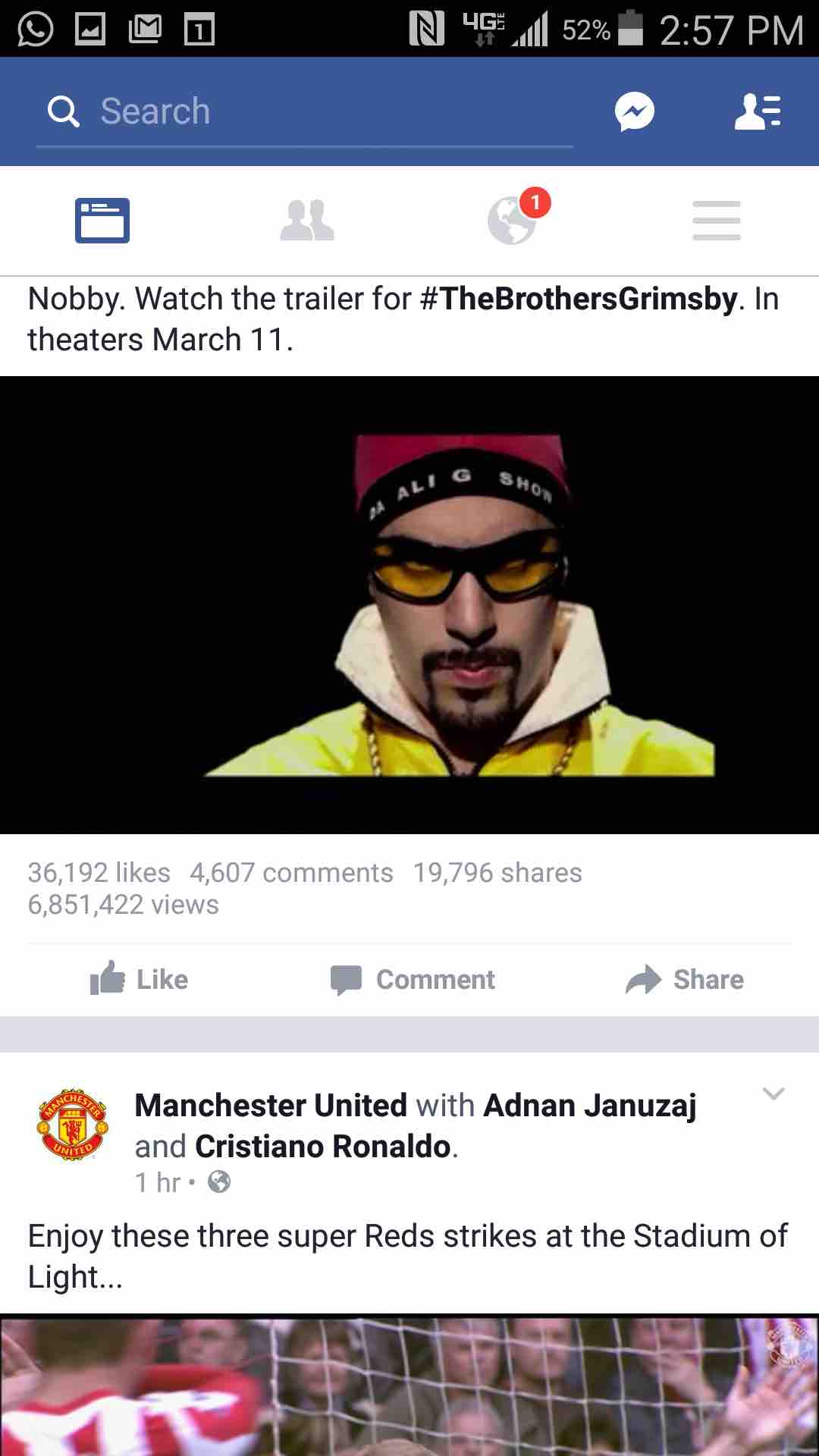}
  \label{fig:facebook}}
\end{subfigure}
\begin{subfigure}[Instagram]{
  \includegraphics[scale=0.062]{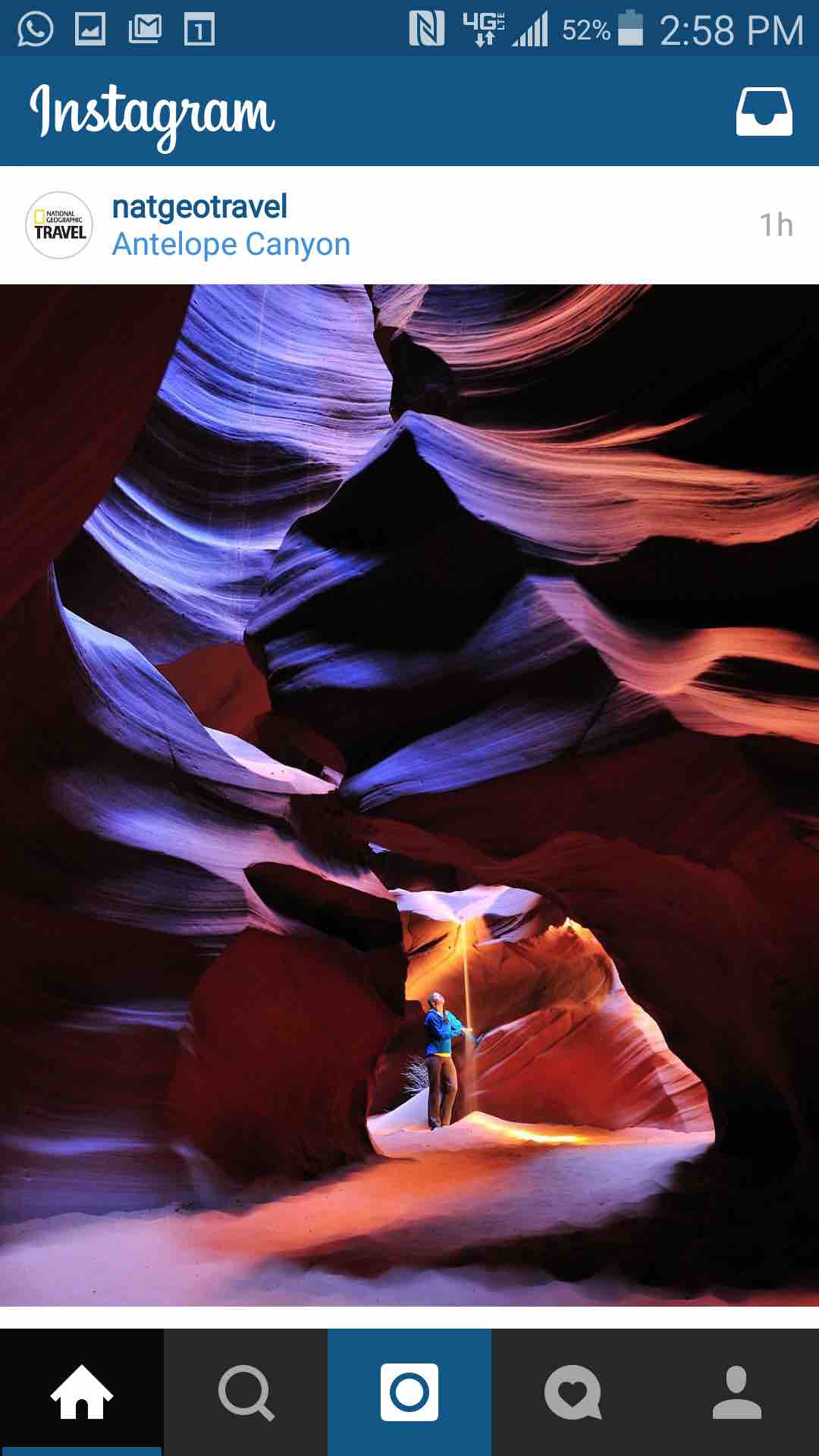}
  \label{fig:instagram}}
\end{subfigure}
\caption{ 
(News) Feed from various social networks. Each of them has
recommendations ranked in some order over multiple slots}
\label{fig:feeds}
\end{minipage}\quad
\begin{minipage}[b]{0.48\textwidth}
\centering
\begin{subfigure}[PYMK]{
  \includegraphics[scale=0.062]{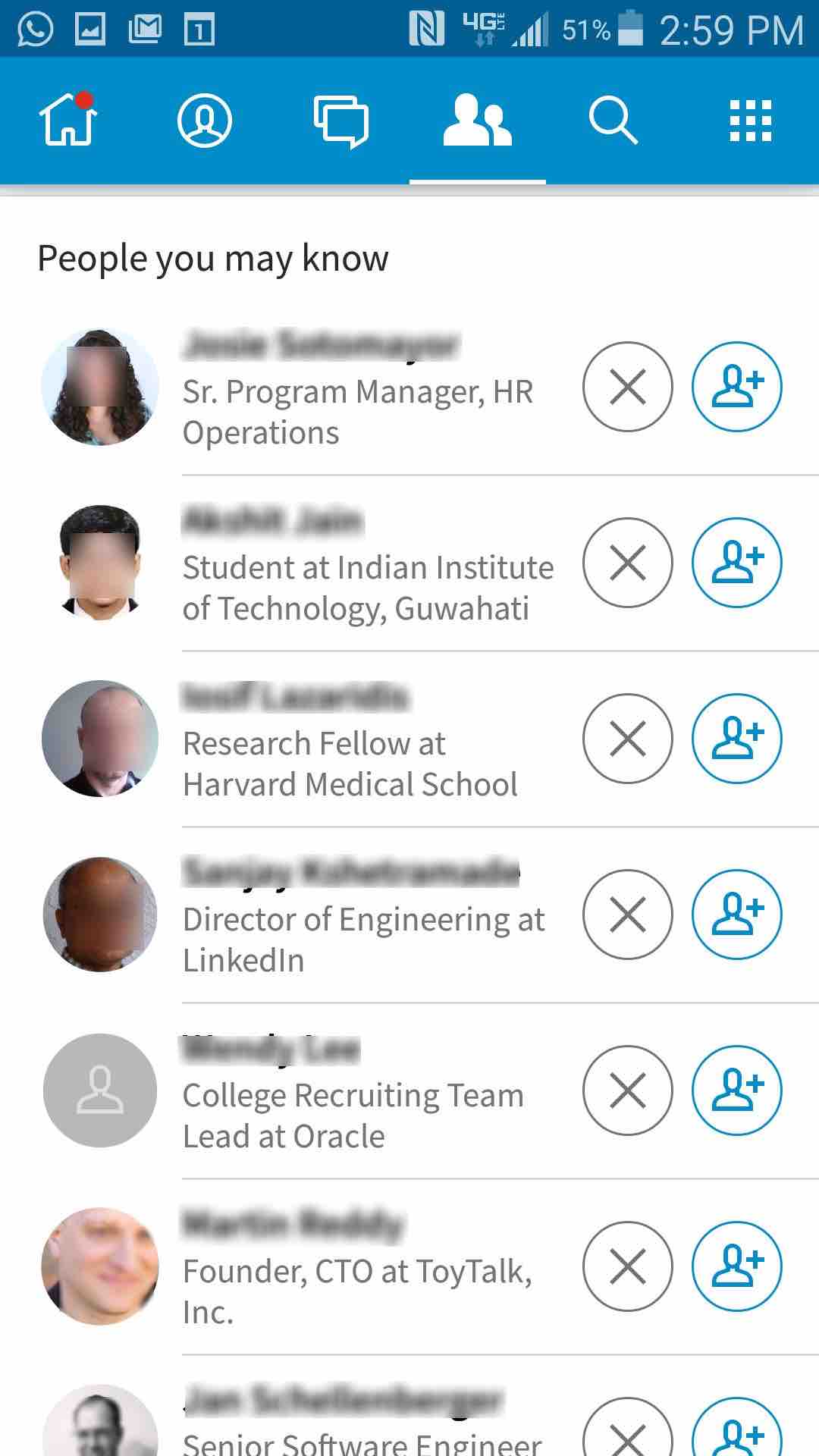}
  \label{fig:pymk}}
\end{subfigure}
\begin{subfigure}[Email with Job postings]{
  \includegraphics[scale=0.062]{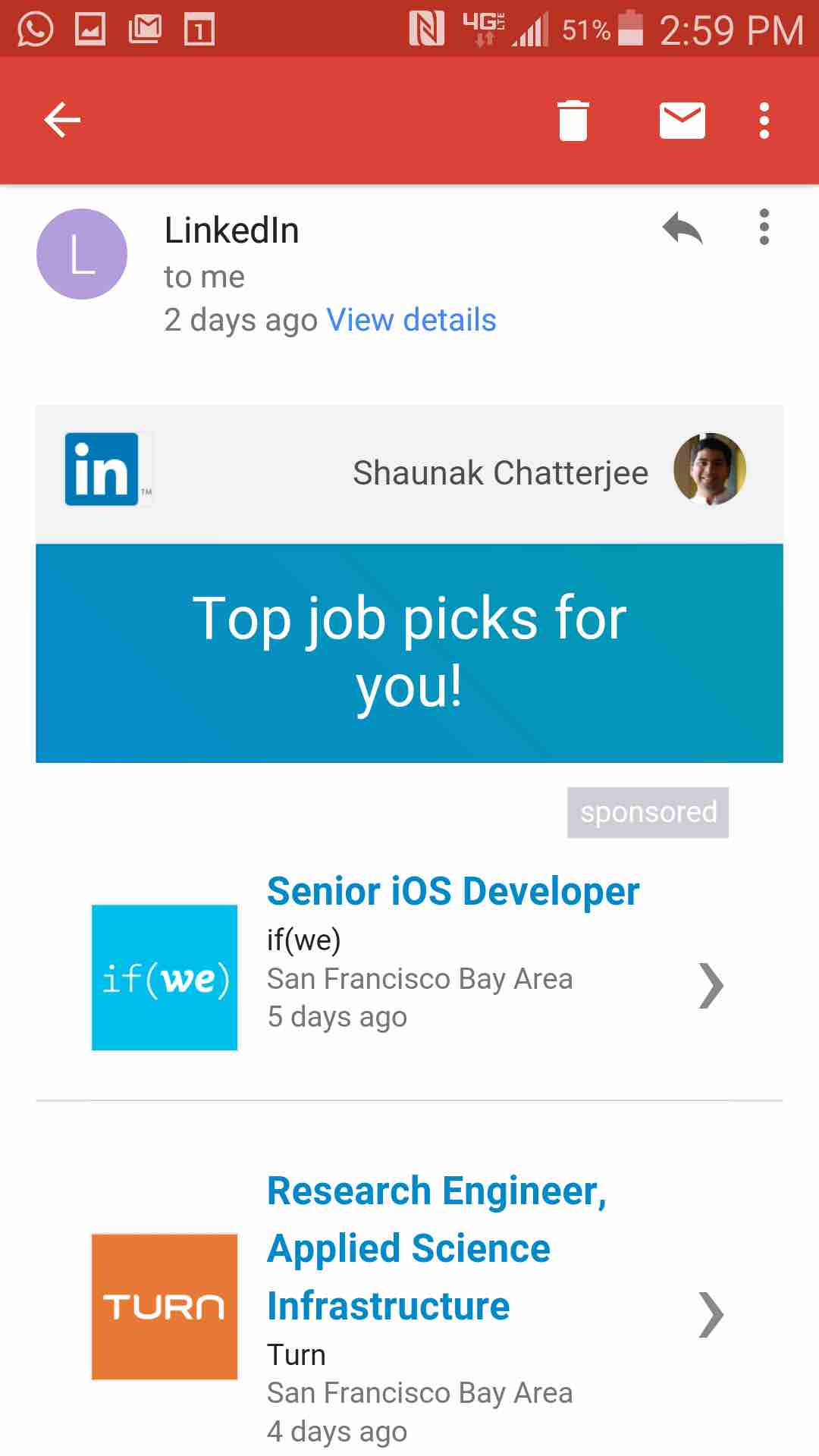}
  \label{fig:email}}
\end{subfigure}
\caption{ 
Other applications including People Recommendation and Emails with job
recommendations, having a ranking of items over multiple slots.}
\label{fig:apps}
\end{minipage}
\end{figure*}


However most of the practical, scalable approaches \textbf{consider
generating scores for each candidate item for a single slot, rather
than considering a global optimization for the entire list of
candidate items}. As a result, they fail to take into account the
interaction effect these items can have on the user while occuring in
multiple slots. Consider the example of a user being shown two similar
items, both of which are separately very relevant to her, but if we
rank them consecutively at the top of her feed, the user is less
likely to click on the second item and the user experience
suffers. These ideas can also be extended to cover approaches like
impression discounting of a group of items (which governs that a user
should not see too many impressions of the same type on consecutive
visits to the feed) as well as diversity of items (which governs that
items of the same type should not occur together). The importance of context of an
item gets magnified on the mobile screen due to the limited attention span and display space.


In this paper, we try to address this class of problems by coming up
with a formulation which allows for this MOO trade off to
be set up as a constrained optimization problem taking into account
interactions between multiple slots along the lines
of \cite{agarwal2012personalized}. Our main contributions can be
summarized as following:
\begin{itemize}
 \item We provide an efficient solution for the MOO problem on the
 feed (without assuming multi-slot dependency) by setting it up as a
 QP and optimizing the dual of the QP. Although this method has been
 explored in \cite{agarwal2012personalized}, we can handle much more
 general classes of constraints and provide a trick of efficiently
 obtaining the primal optimum from the dual by evaluating a cheap
 projection. This enables us to come up with a deterministic serving
 plan for serving items on the feed in a sorted order to maximize some
 notion of user engagement while also keeping other business metrics
 above accepted thresholds.

 \item We formulate the problem of MOO in the multi-slot case,
 where we allow for interactions between the different items. We identify specific interaction models, mathematical conditions and assumptions which result in the problem being a QCQP. Without those, the problem is much more computationally intractable. We show that solving this QCQP results  in a far superior feed than without modeling the
 multi-slot dependence.

\item One of the big challenges is to solve a large scale QCQP in an
efficient manner. We devise an algorithm which approximates the
quadratic constraints by a set of linear constraints, thus converting the problem
into a quadratic program (QP) whose solution is a pretty close
approximation for the original optimum. We show results of convergence
as well as experiments comparing our algorithm to existing QCQP solvers
and how fast it can scale to large scale data. To the best of our
knowledge, this technique is new and has not been previously explored
in the optimization literature.
\end{itemize}

The rest of the paper is structured as follows. In Section \ref{sec:problem_def}, we introduce the problem of feed ranking as a multi-objective optimization problem. Section \ref{sec:soln_prob} gives an efficient solution of the posed problem. In Section \ref{sec:modeling_interaction}, we describe the interaction model and the several characterizations which translate the original QP into a QCQP. We develop a new technique to solve the large scale QCQP in Section \ref{sec:sol_qcqp}. Experimental results follow in Section \ref{sec:expt_results} and we conclude with some discussion in Section \ref{sec:disc}.



%% file: Problem_Defn.tex

We begin by introducing some notation which we use throughout the
paper. Let $i = 1,\ldots, n$ denote the $i$-th user, $j = 1, \ldots,
J$ denote the $j$-th item to be shown to the user and $k = 1,\ldots,
K$ denote the $k$-th slot on the feed. Let $x_{ijk}$ be the
serving probability of item $j$ at slot $k$ to user $i$ and $p_{ijk}$
denote the probability of clicking item $j$ at slot $k$ by user $i$
conditional on the fact that $j$-th item was shown at position $k$ to
user $i$.

We allow both organic and sponsored items on the feed. Let $c_{ijk}$ be the dollar value associated to clicking item $j$ and $\Jcal_s$ denote the set of sponsored items. We
will assume throughout that the revenue generated from a sponsored item does
not depend on the position of the item, i.e. $c_{ijk} = c_j > 0$ for all
$i, k$ and $j \in \Jcal_s$. We also assume that $c_{i,j,k}= 0$ for $j \not\in \Jcal_s$. We introduce the vector notation for
$x_{ijk}, p_{ijk}, c_{ijk}$ as $\xb, \pb$ and $\cbb$ respectively. Further
let $\dolb = \pb\cdot \cbb$ where $\cdot$ denotes the element wise
multiplication of the two vectors.

We set up the problem of feed ranking in the presence of sponsored
items as a constrained optimization problem. This formulation has been
familiar in the literature \cite{agarwal2015constrained, agarwal2012personalized}
 and targets mainly at maximizing the overall
click-through rate (CTR) while simultaneously trying to maximize
revenue by showing sponsored content on the feed. These aspects can
often conflict with each other (since organic items on the feed can
have higher CTR than sponsored items) which leads to the modified
problem of maximizing the CTR under constraints imposed by business
rules, in terms of maintaining the revenue beyond a certain threshold,
showing more than a certain number of posts of a specific type
(e.g. news articles) among others. 

These requirements can be met by setting up a multi-objective
optimization (MOO) problem. Specifically, we wish to maximize the expected
clicks on the feed while keeping the revenue above a certain level and
the total number of impressions at a certain level. Under the above
notation, the expected number
of clicks is given by 
\begin{align*}
  \mathbb{E}(\text{Clicks}) = \sum_{i,j,k} x_{ijk}p_{ijk}.
\end{align*}
while the two constraints can be written as
\begin{align*}
\sum_{i,j,k} x_{ijk}p_{ijk}c_{ijk} &\geq R \\
\sum_{j \in \Jcal_I} \sum_{i,k} x_{ijk} &\geq I \qquad 
\end{align*}
where $R$ and $I$ refer to business specified thresholds while $\Jcal_I$
defines a subset of items whose impressions are required to be more
than a specified amount. Further, we know that since $x_{ijk}$ is a
probability we have $0 \leq x_{ijk} \leq 1$. Moreover, we know that
for each slot $k$ we show an item. Hence, $\sum_{j=1}^J x_{ijk} = 1$
for all $i, k$. And if we sum across $k$ we get the probability of
showing the $j$-th item. Thus we have $ 0 \leq \sum_{k=1}^K x_{ijk}
\leq 1$ for all $i,j$. Combining all of these together we can
formulate the problem of maximizing the clicks on the multi-slot case
as follows.
\begin{equation}
\label{orig_problem}
\begin{aligned}
& \underset{x}{\text{Maximize}} & & \sum_{i,j,k} x_{ijk}p_{ijk} \\
& \text{subject to} & &\sum_{i,j,k} x_{ijk}p_{ijk}c_{ijk} \geq R \\
&  &  &\sum_{i,j,k} x_{ijk} d_{ijk} \geq I \\
& & & 0 \leq x_{ijk} \leq 1 \; \forall \; i,j,k \\
& & & \sum_{j=1}^J x_{ijk} = 1, \; \forall i,k \\
& & & 0 \leq \sum_{k=1}^K x_{ijk} \leq 1 \; \forall i,j,
\end{aligned}
\end{equation}
where $d_{ijk} = 1$ for $j \in \Jcal_I$ and zero otherwise. 

Infact, for each user $i$ the last three constraints create a local
constraint and we rewrite the equations as $\Kb_i \bsx_i \leq
b^i$. Replacing the maximization to a minimization problem (to exploit
the convexity of the objective) and introducing a strongly convex
regularization term for ease of optimization (See
\cite{agarwal2012personalized} for details) we get the following
optimization problem,

\begin{equation}
\label{eq:orig_qp}
\begin{aligned}
& \underset{x}{\text{Minimize}} & & -\xb^T\pb + \frac{\gamma}{2}\xb^T\xb \\
& \text{subject to} & & \xb'\dolb \geq R \\
&  &  &\xb'\db \geq I\\
& & & \bsx_i \in \Kcal_i \; \forall \; i
\end{aligned}
\end{equation}
where $\xb = (\bsx_1, \ldots, \bsx_n)$ and $\Kcal_i$ denotes the convex set created by the linear equations, $\Kb_i \bsx_i \leq b^i$.

The main hurdle in optimizing \eqref{eq:orig_qp} is the prohibitively
expensive cost arising as the number of users $n$ is in the order of
millions for most real life web applications. Thus, instead of solving
it directly we try to evaluate the dual variables corresponding to the
first two global constraints as
in \cite{agarwal2012personalized}. In \cite{basu_user}, the authors
show that if we have the optimal dual variables then it possible to
get the optimal primal solution in an online setting efficiently by
solving a small quadratic programming problem.

%% file: Problem_Soln.tex
Instead of solving the optimization problem \eqref{eq:orig_qp} directly,
we employ a two step procedure as in
\cite{agarwal2012personalized}. In the first stage we solve the dual
problem of \eqref{eq:orig_qp} to get optimal dual variables.  Using
the dual variables we obtain the primal solution through a neat
conversion trick. Finally, we describe how to get a deterministic
serving plan from the probabilistic optimal primal solution. Each step
of this procedure is described below.

\subsection{Dual Solution}
The local constraints $\bsx_i \in \Kcal_i$ for all $i = 1, \ldots, n$ can be
combined together to get a constraint like $\Kb\xb \leq \bb$, where $\Kb =
diag(\Kb_i, i = 1,\ldots, n)$. Thus with this notation we can transform
the problem \eqref{eq:orig_qp} as
\begin{equation}
\label{dual_sol}
\begin{aligned}
& \underset{x}{\text{Minimize}} & & -\xb^T\pb + \frac{\gamma}{2}\xb^T\xb\\
& \text{subject to} & & \xb^T\dolb \geq R \\
&  &  &\xb^T\db \geq I\\
& & & \Kb\xb \leq \bb
\end{aligned}
\end{equation}
Writing the Lagrangian of the above problem, we have
\begin{align*}
L(\xb, \mu_0, \mu_1, \etab ) &= -\xb^T\pb + \frac{\gamma}{2}\xb^T\xb + \mu_0(R -
\xb^T\dolb)\nonumber \\ & \;\;\; + \mu_1(I - \xb^T\db) + \etab^T(\Kb\xb - \bb).
\end{align*}
Finding the minimum with respect to $x$ we see get,
\begin{align*}
\hat{\xb} = \frac{1}{\gamma}\left(\mu_0\dolb + \mu_1\db + \pb - \Kb^T\etab \right.).
\end{align*}
Plugging this back into the Lagrangian we get,
\begin{align*}
L(\hat{\xb}, \mu_0, \mu_1, \etab ) &= \frac{\gamma}{2}\hat{\xb}^T\hat{\xb}
-\hat{\xb}^T(\mu_0\dolb + \mu_1\db + \pb - \Kb^T\etab )\\ &\;\;\; + \mu_0R +
\mu_1I - \etab^T\bb \\ & = -\frac{\gamma}{2}\hat{\xb}^T\hat{\xb} + \mu_0R +
\mu_1I - \etab^T\bb.
\end{align*}
Writing $\xib = (R,I,-\bb^T)^T, \;\; \Ab = (\dolb : \db : -\Kb^T)$ and $\yb =
(\mu_0, \mu_1, \etab)$ the above problem simplifies to
\begin{equation}
\label{dual_sol_simplified}
\begin{aligned}
& \underset{\yb}{\text{Maximize}} & & \xib^T\yb -\frac{1}{2\gamma}(\pb +
  \Ab\yb)^T(\pb + \Ab\yb)\\ & \text{subject to} & & \yb \geq 0
\end{aligned}
\end{equation}
which can be further simplified to 
\begin{equation}
\label{dual_sol_final}
\begin{aligned}
& \underset{y}{\text{Minimize}} & & \frac{1}{2}\yb^T\Mb\yb - \yb^T(\xib -
  \Ab^T\pb/\gamma)\\ & \text{subject to} & & \yb \geq 0
\end{aligned}
\end{equation}
where $\Mb = \Ab^T\Ab/\gamma$. We solve this by the operator splitting
algorithm \cite{parikh2014block} to obtain $\mu_0$ and $\mu_1$. 

\begin{remark} 
  We cannot apply block splitting algorithm \cite{parikh2014block} to
  the dual problem since the objective function is no longer
  separable. We may apply it to the primal but obtaining the dual
  solution may be not be possible in all situations. For details, we
  refer to \cite{basu_user}.
\end{remark}

The ability to solve this problem on a large scale completely depends
on the sparsity of matrix $\Mb$.

\begin{definition}
Let $\Ab$ be any $m\times n$ matrix. We define the sparsity ratio,
$\psi(\Ab)$ as $k$ if the number of non-zero entries in the matrix is
$kmn$.
\end{definition}

From previous experimentation we were able to solve the dual problem
having $n$ variables through operator splitting if $\psi(\Mb)$ is of
order $O(1/n)$. We shall show in this setup too, $\psi(\Mb)$ is of the
same order.


\begin{lemma}
\label{lem:sparse}
Consider the dual problem for the multi-slot optimization problem as
given in \eqref{dual_sol_final}. Let $\Jcal_s$ and $\Jcal_I$ denote
the sets of sponsored and impression important items. If there are $K$
slots, $J$ total items and $n$ users, then,
\[
\psi(\Mb) = \frac{1 + n(J + \beta + K(3 + \beta) + 7JK)}{(1+ nJ + nK + nJK)^2}.
\]
where $\beta = |\Jcal_s| + |\Jcal_I|$.
\end{lemma}
\begin{proof}
Note that, using the previous notation we can write,
\begin{align*}
\Mb = \frac{1}{\gamma} \Ab^T\Ab = \frac{1}{\gamma} \begin{bmatrix}
    \dolb^T\dolb & \dolb^T\db & -\dolb^T\Kb^T\\
    \db^T\dolb & \db^T\db & -\db^T\Kb^T  \\
    -\Kb\dolb & -\Kb\db & \Kb\Kb^T
\end{bmatrix}.
\end{align*}
Without loss of generality we assume that
$\Jcal_s \cap \Jcal_I \neq \emptyset$. Then the top-left $2 \times 2$
principal sub matrix of $M$ contains 4 non-zero terms. Otherwise only
the diagonals are positive and we only lose a count of 2 to the total
number of non-zero terms. Note that for each $i = 1, \ldots, n$ we can
write $\Kb_i^T$ as
\begin{align*}
\Kb^T_i = \begin{bmatrix}
    \Ib_{JK} & -\Ib_{JK} & \Bb^T & -\Bb^T & \Cb^T & -\Cb^T
\end{bmatrix},
\end{align*}
where, 
\begin{align*}
\Bb = \begin{bmatrix}
    \Ib_{K} & \Ib_{K} & \ldots & \Ib_{K}
\end{bmatrix}_{K \times JK},\qquad
\Cb = \Ib_{J} \otimes 1_{K}^T,
\end{align*}
and $\Ib_n$ denotes the identity matrix of dimension $n$ and $\otimes$
denote the Kronecker product. Note that the matrices in $\Kb_i$
denotes the constraints for the $i$-user of optimization
problem \eqref{orig_problem}. Let $\mathrm{card}(\Ab)$ denotes the
number of non-zero entries in a matrix $\Ab$. Now it is easy to see
that,
\begin{align}
\label{eq:s1}
\mathrm{card}(\Kb_i \dolb_i) &= 2( K + |\Jcal_s| + K|\Jcal_s|)\\
\label{eq:s2}
\mathrm{card}(\Kb_i \db_i)  &= 2( K + |\Jcal_I| + K|\Jcal_I|).
\end{align}
To complete the proof, we note that,
\begin{align}
\label{eq:s3}
\mathrm{card}(\Ib_{JK}) &= \mathrm{card}(\Bb) = \mathrm{card}(\Cb) = JK,\\
\label{eq:s4}
\mathrm{card}(\Bb\Bb^T)  &= K, \;\;\mathrm{card}(\Cb\Cb^T)  = J, \;\;\mathrm{card}(\Bb\Cb^T)  = JK.
\end{align}
Now, using the structure is $\Kb\Kb^T$, and equations \eqref{eq:s3} and \eqref{eq:s4}, we get
\begin{align}
\label{eq:s5}
\mathrm{card}(\Kb\Kb^T) = 4n(J + K + 7JK).
\end{align}
The result follows from using \eqref{eq:s1}, \eqref{eq:s2}
and \eqref{eq:s5} and observing that the total number of entries in
the matrix $\Mb$ is $(2 + 2nK + 2nJ + 2nJK)^2$.
\end{proof}

From Lemma \ref{lem:sparse}, it is easy to see that $\psi(\Mb) =
O(1/(nJK))$. Thus, for example if we are able to solve the single slot
for 2 million users, having $K = 5$ and $J = 10$ we can solve this
multi-slot problem at the same time and same accuracy for 40000 users.

\subsection{Dual to Primal Trick}
The main idea behind this technique is to exploit the KKT conditions
of the problem, so that we can write the primal solution as a function
of the dual solution and the input parameters. To explain this
technique, let us start by rewriting the problem given in
\eqref{eq:orig_qp} as
\begin{equation}
\label{eq:dual_trick}
\begin{aligned}
& \underset{\xb}{\text{Minimize}} & & -\xb^T\pb + \frac{\gamma}{2}\xb^T\xb +
  \mathds{1}_{\Kcal_1 \times \Kcal_2 \times \ldots \times \Kcal_n}(\xb) \\ & \text{subject
    to} & & \xb^T\dolb \geq R, \\ & & &\xb^T\db \geq I,
\end{aligned}
\end{equation}
where $\Kcal_1 \times \Kcal_2 \times \ldots \times \Kcal_n$ denotes the domain of
$\xb$ and $\mathds{1}_{\Ccal}(\xb) = 0$ if $\xb \in \Ccal$ and $\infty$ otherwise. Introducing the dual variables $\mu_0$ and $\mu_1$ for the two
constraints we can write the Lagrangian as follows.
\begin{align*}
L(\xb, \mu_0, \mu_1) = -\xb^T\pb + \frac{\gamma}{2}\xb^T\xb + \mu_0(R - \xb^T\dolb) +
\mu_1(I - \xb^T\db),
\end{align*}
for $\xb \in \Kcal_1 \times \Kcal_2 \times \ldots \times \Kcal_n$. From this it is
easy to see that the optimal solution satisfies
\begin{align}
\label{primal_proj}
\xb^{opt} = \Pi_{\Kcal_1 \times \Kcal_2 \times \ldots \times \Kcal_n}
\left(\frac{1}{\gamma}(\mu_0\dolb + \mu_1\db + \pb)\right),
\end{align}
where $\Pi_{\mathcal{C}}(\cdot)$ denotes the projection function onto
$\mathcal{C}$. Note that, since we have a product domain, we can write
this as
\begin{align*}
\bsx^{opt}_i = \Pi_{\Kcal_i} \left(\frac{1}{\gamma}(\mu_0\dolb + \mu_1\db +
\pb)_i\right)\;\; \forall \;\; i = 1,\ldots,n.
\end{align*}
Thus once we know $\mu_0$ and $\mu_1$ by
solving \eqref{dual_sol_final}, $x_i$ can be obtained simply by the
projection onto $\Kcal_i$ which is a small quadratic problem and can
be solved fast and efficiently.

\subsection{Deterministic Serving Plan}
Here we present the algorithm to generate the serving plan for user
$i$ using the optimal probabilistic solution $x_{ijk}$. We know for
each slot $k$ the items follow a multinomial distribution with
probability $x_{ijk}$ for $j = 1, \ldots, J$. While serving we have
the added criteria that no two items can be repeated. Thus for each
slot $k$ we sample from $Mult( \{x_{ijk}\}_{j=1}^J)$ and we resample
if we get the same item for two slots. This guarantees that that
serving plan obeys the optimal serving distribution. The detailed
steps are written out in Algorithm \ref{algo:plan1}.

\begin{algorithm}
\caption{Deterministic Serving Plan}\label{algo:plan1}
\begin{algorithmic}[1]
\State \text{Input : Optimal primal solution $\xb$}
\State \text{Output : Deterministic serving plan for each user}
\State \text{Set $S$ empty $n \times K$ matrix} 
\For{$i = 1:n$}
\State $Sample = \emptyset$
\For{$k = 1:K$}
\State \label{sample-step} Pick $j \sim Mult( \{x_{ijk}\}_{j=1}^J)$
\If{$j \in Sample$}
\State Go to step \ref{sample-step}.
\Else
\State Add $j$ to Sample.
\EndIf
\EndFor
\State $S[i, ] = Sample$
\EndFor
\State \Return $S$
\end{algorithmic}
\end{algorithm}
\vspace{5em}

%% file: Modeling_Interaction.tex
The previous sections outline a scheme to display feed contents using
a constrained optimization formulation assuming the items act
independently of each other. However in most situations, the feed
content has an interaction effect. In other words, the user's
probability of clicking on a post might change depending on the type
of items he has seen in his feed until now, particularly if he has
been displayed items from the same content group.  
In order to avoid this problem, most feed algorithms use a mixer to
mix content from different channels. In practice, for each channel or
group we have a sorted content list based on a single slot
optimization problem. The output of each channel is then mixed to
create the final feed. However, doing this clearly results in a loss
of the optimal ranking of the feed content and a sub-par feed
experience. 

One of the ways of tackling this problem is to set it up as a
multi-slot optimization problem so that we can directly optimize for
the best feed ranking under constraints while also respecting the
dependency structure among multiple slots. Note that the same idea
also comes in useful for tackling problems related to impression
discounting or diversity on the feed, which influence the ranking



In this section, we develop a technique to solve the dependency
problem as a multi-slot optimization problem instead of using ad-hoc
mixing algorithms to individual channels. Following the notation in
the previous sections, we consider the following generic optimization
problem which encompasses the objectives described before.
\begin{equation}
\label{eq:dual_sol_easy}
\begin{aligned}
& \underset{\xb}{\text{Minimize}} & & -\xb^T\pb + \frac{\gamma}{2}\xb^T\xb\\ &
  \text{subject to} & &\xb^T\rb \leq P\\ & & & \Kb\xb \leq \bb
\end{aligned}
\end{equation}
Note that $\pb$ and $\rb$ are the only two parameters in the problem
and we currently assume that they are independent of each other. (For
the more generalized formulation, see Section \ref{sec:dep_p_r}). A
lot of practical problems can be posed as above. Some examples include
multi-slot feed where $-\rb$ can be thought of as a substitute of
$\dolb$ from \eqref{eq:orig_qp}, multi-slot product email updates with
$\pb$ denoting probability of clicks and $\rb$ denoting the
probability of complaints on sending the email, the
People-You-May-Know application which is a recommended list of people
you can professionally connect with, where $\pb$ can denote the
probability of clicks while $\rb$ can represent the probability of
dismissing the recommendation. 

We model the interaction
effect as follows
\begin{align}
\label{eq:interaction}
\pb &= -\Qb_p\xb \qquad\text{ and }\qquad \rb = \Qb_r \xb.
\end{align}
where $\Qb_p$ and $\Qb_r$ are some positive definite matrices (exact
choice of which is discussed in Section \ref{sec:interact}).  Using
\eqref{eq:interaction} in \eqref{eq:dual_sol_easy} we re-write the problem as,
\begin{equation}
\label{dual_sol_QCQP}
\begin{aligned}
& \underset{x}{\text{Minimize}} & &  \xb'\left(\Qb_p + \frac{\gamma}{2}\Ib \right) \xb \\
& \text{subject to}&  & \xb'\Qb_r\xb  \leq P\\
& & & \Kb\xb \leq \bb
\end{aligned}
\end{equation}
Note that this makes the problem much more complicated since we
transform a quadratic programming problem (QP) to a quadratically
constrained quadratic program (QCQP) a general version of which is
actually NP-hard \cite{boyd2004convex}. Moreover the choice of $\Qb_p$
and $\Qb_r$ is extremely important since convexity of the problem
depends on $\Qb_p, \Qb_r$ being positive definite.

\subsection{Choice of $\Qb_p$ and $\Qb_r$}
\label{sec:interact}
The choice of these matrices is a hard problem in practice since we do
not know the exact dependency structure of $\pb$ and $\rb$ on
$\xb$. However, due to the presence of enormous empirical data for
clicks as well as complaints and dismisses, estimation of
$\Qb_p, \Qb_r$ is not a hard problem, if there exists a parametric
form. Since $\Qb_p$ and $\Qb_r$ can have similar parametric form,
without loss of generality we explain the form through $\Qb_p$.

Throughout we assume $\Qb_p = Diag( \Qb_i, i = 1, \ldots, n)$,
i.e. the probabilities across the users are independent and each
$\Qb_i$ is a $JK
\times JK$ dimensional positive definite matrix which we parametrize
below. We assume that all $K$ slots are viewable by a user, and the
chance of an event at any slot depends on all the rest. Moreover we
hide the $i$ in the subscript for notational simplicity.

Let us begin by introducing some more notation. For any $j \in \{1,
\ldots, J\}$, let $\pb_j, \xb_j$ denote $K$-dimensional vectors
comprising of elements $\{p_{jk}\}_{k=1}^K$ and $\{x_{jk}\}_{k=1}^K$
respectively. Let $\Qb_{jj'}$ denote the $K \times K$ dependence
matrix between event $j$ and observation $j'$. This leads to the
following expression
\begin{align*}
\pb = \left(
  \begin{tabular}{c}
  $\pb_1$\\
  $\pb_2$ \\
  $\vdots$\\
  $\pb_J$ 
  \end{tabular}
\right) = 
\begin{bmatrix}
    \Qb_{11} & \Qb_{12} & \Qb_{13} & \dots  & \Qb_{1J} \\
    \Qb_{21} & \Qb_{22} & \Qb_{23} & \dots  & \Qb_{2n} \\
    \vdots & \vdots & \vdots & \ddots & \vdots \\
    \Qb_{J1} & \Qb_{J2} & \Qb_{J3} & \dots  & \Qb_{JJ}
\end{bmatrix}
 \left(
  \begin{tabular}{c}
  $\xb_1$\\
  $\xb_2$ \\
  $\vdots$\\
  $\xb_J$ 
  \end{tabular}
  \right)
\end{align*}

Now we characterize each $\Qb_{jj'}$. Note that for any $j,k$, $p_{jk}$
is the probability of an event conditional of the fact that $x_{jk} =
1$. Thus, it depends on $\xb_j$ only through $x_{jk}$. Hence it is easy
to see that for any $j \in \{1, \ldots, J\}$,
\begin{align}
\label{eq:q_diag}
\Qb_{jj} = \tilde{p}_j \Ib_K
\end{align}
where $\tilde{p}_j$ is the prior event probability conditional on
seeing item $j$ irrespective of its position, and $\Ib_K$ is the $K
\times K$ identity matrix. The values of $\tilde{p}_j$ for $j = 1,
\ldots, J$ can be estimated from empirical data.

Now for any $j \neq j'$, we consider the dependency between the event
corresponding to item $j$ and the observation of $j'$. When we are
considering $p_{jk}$, since we know that $x_{jk} = 1$, we must have
that the contribution from $x_{j'k} = 0$ for $j' \neq j$. Moreover,
that is the only coefficient which is zero, since we let $p_{jk}$
depend on $x_{j'k'}$ for $k' \neq k$. Thus, we have,

\begin{align}
\label{eq:q_off_diag}
\Qb_{jj'} = \begin{bmatrix}
    0 & a_{12} & a_{13} & \dots  & a_{1K} \\
    a_{21} & 0 & a_{23} & \dots  & a_{2K} \\
    \vdots & \vdots & \vdots & \ddots & \vdots \\
    a_{K1} & a_{K2} & a_{K3} & \dots  & 0
    \end{bmatrix}.
\end{align}
We estimate each of the $a_{\ell\ell'}$ using empirical data. Finally
to bring in symmetry in the problem, we assume that
\begin{align}
\label{eq:q_sym}
\Qb_{j'j} = \Qb_{jj'}^T
\end{align}
for all $j$ and $j'$. This condition will often be true for similar items, or for dissimilar items which have similar effects on one another (whether mutually synergistic or antagonistic). However, we also acknowledge that there will be certain cases where this condition will not hold, and that will make the problem non-convex. 

Combining the structure in equations \eqref{eq:q_diag},
\eqref{eq:q_off_diag} and \eqref{eq:q_sym}, we get the complete
parametric form of the matrix $\Qb$.

\subsection{Practical considerations}

If we assume the parametric form of $\Qb$ as given in Section
\ref{sec:interact}, the postive-definiteness of $\Qb$ is a concern since
we estimate each of the $a_{\ell\ell'}$ from the data. An easy fix
to this problem, is to add $\eta \Ib_{JK}$ to $\Qb$ for some appropriate
value of $\eta$. This enforces that $\Qb$ to be positive-definite by
only slightly increasing the elements corresponding to
$\tilde{p}_{j}$. Thus instead of using the estimated $\Qb$ directly we
use,

\begin{align}
\label{eq:tilde_Q}
\tilde{\Qb} := \begin{cases}
\Qb + (-\lambda_1(\Qb) + \epsilon)I_{JK} & \text{ if } \lambda_1(\Qb) < 0\\
\Qb & \text{ otherwise}
\end{cases}
\end{align}
where $\lambda_1(\Qb)$ denotes the minimum eigenvalue of $\Qb$ and
$\epsilon > 0$ is arbitrary.

An advantage of using this technique from a computational point of
view is that since each $\Qb$ is of small dimension ($JK$ being small),
the minimum eigenvalue calculation can be done extremely fast and can
be made highly parallel for different users $i$. Secondly, this
guarantees that $\Qb_p = Diag(\tilde{\Qb}_i, i = 1,\ldots,n)$ is also
positive-definite.

\subsection{Dependence of Parameters}
\label{sec:dep_p_r}
Whenever we are modeling the interaction effect, we have assumed that
vectors $\pb$ and $\rb$ are independent. However, there may arise problems
where that is not the case. For example, if we consider our original
problem \eqref{eq:orig_qp}, then $\rb = \dolb = \pb\cdot\cbb$ and hence they are
clearly not independent. Here, it is clear that if we construct
$\Qb_r = \Qb_p \cdot \cbb$, where $\Qb_p \cdot \cbb$ is a matrix whose $\ell$-th row is
the $\ell$-th row of $\Qb_p$ multiplied by $c_\ell$, then $\Qb_r$ loses
its positive-definiteness since it is not even symmetric. For a general dependency structure, we have the following result.

\begin{theorem}
Suppose we have an optimization problem of the form
\eqref{eq:dual_sol_easy}. Moreover, assume that $\rb = f(\pb)$ for some
function $f$, that maps $\mathbb{R}^n \rightarrow \mathbb{R}^n$ and
let $\Qb_p = Diag(\tilde{\Qb}_i, i = 1,\ldots,n)$ where $\tilde{\Qb}_i$'s
are generated by \eqref{eq:tilde_Q}. If there exists a
positive-definite matrix $\Qb_r$ and a vector $\cbb_r$ such that
\begin{align}
\label{eq:cond1}
\Qb_r \xb  - 2 \Qb_r \cbb_r = f(-\Qb_p \xb)
\end{align} 
for all $\xb$, then there exists a convex QCQP formulation of
\eqref{eq:dual_sol_easy}.
\end{theorem}

\begin{proof}
Assume there exists a positive-definite matrix $\Qb_r$ and $\cbb_r$ such
that such that \eqref{eq:cond1} holds. Thus, using the dependence
structure and modeling the interaction effect, we can write,
\begin{align*}
\xb^T\rb &= \xb^Tf(\pb) = \xb^Tf(-\Qb_p \xb)\\
&= \xb^T\Qb_r\xb - 2\xb^T\Qb_r \cbb_r\\
&= (\xb - \cbb_r)^T \Qb_r (\xb - \cbb_r) - \cbb_r^T\Qb_r \cbb_r.
\end{align*}
Thus we have,
\begin{align*}
\xb^T\rb \leq P \Longleftrightarrow (\xb - \cbb_r)^T \Qb_r (\xb - \cbb_r) \leq P +
\cbb_r'\Qb_r \cbb_r.
\end{align*}
The result follows from observing the fact that since $\Qb_r$ is
positive definite, $\cbb_r'\Qb_r \cbb_r > 0$.
\end{proof}

\begin{remark}
Few simple functions $f$ for which the convex reformulation Theorem
holds include the linear shift operators and the positive scaling
operators. The exact characterization of the function class may be
hard and is beyond the scope this paper.
\end{remark}

%% file: Soln_QCQP.tex
In this section we describe the technique we use to solve the
following optimization problem.
\begin{equation}
\label{eq:QCQP}
\begin{aligned}
& \underset{\xb}{\text{Minimize}} & &  (\xb - \ab)^T\Ab (\xb - \ab) \\
& \text{subject to}&  &(\xb - \bb)^T\Bb(\xb - \bb)  \leq \tilde{b}\\
& & & \Cb\xb \leq \cbb,
\end{aligned}
\end{equation}
where $\Ab, \Bb$ are positive definite matrices. Note that this is a QCQP
in its general form having a single quadratic constraint. The problem as
stated in \eqref{dual_sol_QCQP} is a special case of this
formulation. It is already known that solving a general QCQP is
NP-hard \cite{boyd2004convex}. Before we discuss our methodology, we
discuss few of the common techniques used in literature.

There are two main relaxation techniques that are used to solve a
QCQP, namely, semi-definite programming (SDP) and
reformulation-linearization technique (RLT)
\cite{boyd2004convex}. However both of them introduce a new variable
$\Xb = \xb\xb^T$ so that the problem becomes linear in $\Xb$. They
relax the condition $\Xb = \xb\xb^T$ using different means. However,
in doing so they increase the number of variables from $n$ to
$O(n^2)$. This makes these methods prohibitively expensive for most
large scale methods.

However, based on the Operator Splitting method from \cite{basu_user}
we are capable of solving a large enough QP that scales to web
applications.  This motivated us to try to convert the QCQP into an
approximate QP by linearizing the constraints, which we can then solve
efficiently. To the best of our knowledge this technique is new and
has not been previously explored in the literature. The rest of this
section is devoted to explain the linearization of the quadratic
constraint. We also give several results showing convergence
guarantees to the original problem.

\subsection{QCQP to QP Approximation}
Here we describe the linearization technique to convert the quadratic
constraint into a set of $N$ linear constraints. The main idea behind
this approximation, is the fact that given any convex set in the
Euclidean plane, there exists a convex polytope that covers the
set. 

Let us begin with a few notations. Let $\cp$ denote the optimization
problem \eqref{eq:QCQP}. Let $\xb \in \real^{s}$. Define,
\begin{align}
\label{eq:qc}
\cs := \{ \xb \in \mathbb{R}^s : (\xb - \bb)^T\Bb(\xb - \bb)  \leq \tilde{b} \}.
\end{align}
Let $\partial\cs$ denote the boundary of the ellipsoid $\cs$.

To generate the $N$ linear constraints for this one quadratic
constraint, we generate a set of $N$ points, $\xb_1, \ldots, \xb_N \in
\partial\cs$. The sampling technique to select these points are given
in Section \ref{sec:sampling}. Corresponding to these $N$ points we
get the following set of $N$ linear constraints,
\begin{align}
\label{eq:lin_const}
(\xb - \bb)^T \Bb(\xb_j - \bb) \leq \tilde{\bb} \qquad \text{ for } j = 1, \ldots, N.
\end{align}
If we look at it geometrically, it is not hard to see that each of
these linear constraints are just tangent planes to $\cs$ at $\xb_j$
for $j = 1, \ldots, N$. Figure \ref{fig:qcqp_qp_lin} shows a set of
six linear constraints for a ellipsoidal feasible set in two
dimensions. Thus, using these $N$ linear constraints we can write the
approximate optimization problem, $\cp(N)$, as follows.

\begin{equation}
\label{eq:QP_via_lin}
\begin{aligned}
& \underset{x}{\text{Minimize}} & &  (\xb - \ab)^T\Ab (\xb - \ab) \\
& \text{subject to}&  &(\xb - \bb)^T\Bb(\xb_j - \bb)  \leq \tilde{b}\qquad \text{ for } j = 1, \ldots, N \\
& & & \Cb\xb \leq \cbb.
\end{aligned}
\end{equation}
Thus, instead of solving $\cp$, we solve $\cp(N)$ for a large enough
value of $N$. Note that as we sample more points, our approximation
gets more and more accurate to the original solution.

\begin{figure}[!h]
\centering
\includegraphics[scale = 0.6]{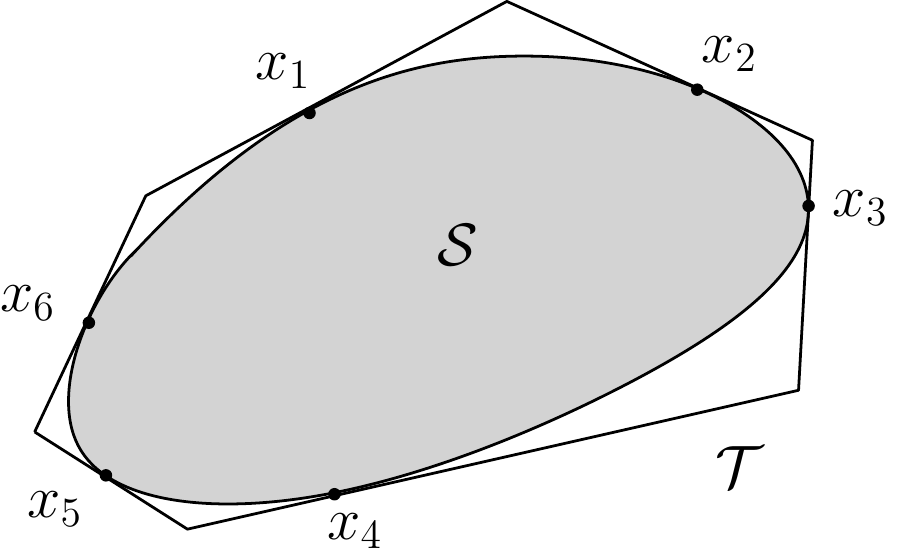}
\caption{\label{fig:qcqp_qp_lin} Converting a quadratic constraint
  into linear constraints. The tangent planes through the 6 points
  $\xb_1, \ldots, \xb_6$ create the approximation to $\cs$.}
\end{figure}

\subsection{Sampling Scheme}
\label{sec:sampling}
Note that the accuracy of the solution of $\cp(N)$ completely depends
on the choice of the $N$ points. The tangent planes to $\cs$ at those
$N$ points create a cover of $\cs$. Throughout this section, we are
going to use the notion of a bounded cover, which we define as
follows.

\begin{definition}
Let $\ct$ be the convex polytope generated by the tangent planes to
$\cs$ at the points $\xb_1, \ldots, \xb_N \in \partial\cs$. $\ct$ is
said to be a bounded cover of $\cs$ if,
$$ \sup_{\tb \in \ct} d(\tb, \cs) < \infty,$$ where $d(\tb, \cs) =
\inf_{\xb \in \cs} \norm{ \tb - \xb }$ and $\norm{\cdot}$ denotes the
Euclidean distance.
\end{definition}

The first result shows that we need a certain minimum number of points
to get a bounded cover. 

\begin{lemma}
\label{lem:min_points}
Let $\cs$ be an $s$ dimensional ellipsoid as defined in
\eqref{eq:qc}. Then we need at least $s+1$ points on $\partial\cs$ to
be get a bounded cover.
\end{lemma}

\begin{proof}
Note that since $\cs$ is a compact convex body in $\real^s$, there exists a $s$-dimensional simplex $T$ such that $ \cs \subseteq T$. We can always shrink $T$ such that each edge touches $\cs$ tangentially. Since there are $s+1$ faces, we will get $s+1$ points whose tangent surface creates a bounded cover. 

To complete the proof we need to show that we cannot create a bounded cover using only $s$ or fewer points. Consider any set of $s$ points, $\xb_1, \xb_2, \ldots, \xb_s$ on $\partial\cs$. The equation of their tangent planes are,
\begin{align*}
(\xb - \bb)^T \Bb(\xb_j - \bb) = \tilde{b} 
\end{align*}
for $j = 1, \ldots, s$. Note that we can rewrite this as $\Ab\xb = \gammab$, where each row $\ab_j = (\xb_j -  \bb)^T\Bb$ and $\gamma_j = \tilde{b} + \bb^T\Bb(\xb_j - \bb)$ for $j = 1, \ldots, s$. 

Without loss of generality we can assume that $\gammab \in \cc(\Ab)$, otherwise the system of linear equations is inconsistent and this case, it is easy to see that the tangent planes do not create a compact polytope.
Now if $\Ab$ is not of full rank, there exists a continuum of solutions to $\Ab\xb = \gammab$. Hence, the polytope is not bounded. Finally if $\Ab$ is invertible, then there exists a unique solution to the system of equations, call it $\xb_0$. Since all the planes intersect at a single point, the polytope is divergent. Thus, it is not possible to construct a bounded cover with only $s$ points. Similar proof holds for fewer than $s$ points. This completes the proof of the lemma.
\end{proof}

Lemma \ref{lem:min_points} states that we need a minimum of $s + 1$
points to create the bounding cover $\ct$, but it does not guarantee
the its existence. Moreover, we wish to pick our points in a way such
that the constructed $\ct$ is close as possible to $\cs$, thus having
a better approximation. Formally, we introduce the notion of optimal
bounded cover.

\begin{definition}
$\ct^* = \ct(\xb_1^*, \ldots, \xb_n^*)$ is said to be an optimal bounded cover, if  
\begin{align*}
\sup_{\tb \in \ct^*} d(\tb, \cs)  \leq \sup_{\tb \in \ct} d(\tb, \cs)
\end{align*}
for any bounded cover $\ct$ generated by any other $n$-point
sets. Moreover, $\{\xb_1^*, \ldots, \xb_n^*\}$ are defined to be the
optimal $n$-point set.
\end{definition}

Note that we can think of the optimal $n$-point set as that set of $n$
points which minimizes the maximum distance between $\ct^*$ and
$\cs$. It is not hard to see that the optimal $n$-point set on the
unit circle are the $n$th roots of unity, unique upto rotation.

\subsubsection{Reisz Energy and Equidistribution}
\label{sec:reisz}
It has been shown that the $n$th roots of unity minimize the discrete Riesz energy for the unit circle \cite{gotz}. Riesz energy of a point set $A_n = \{\xb_1, \ldots, \xb_n\}$ is defined as
\begin{align}
\label{eq:riesz}
E_s(A_n) := \sum_{i \neq j = 1}^n \norm{\xb_i - \xb_j}^{-s}
\end{align}
for positive real parameter $s$. There is a vast literature on Riesz energy and its association with ``good" configuration of points. In fact, we can associate the optimal $n$-point set to the set of $n$ points that minimize the Riesz energy on $\cs$. It is well known that the measures associated to the optimal point set that minimizes the Riesz energy on $\cs$ converges to the normalized surface measure of $\cs$. For sake of compactness we do not go deeper into this. For more details see \cite{grabner, hardin2005minimal} and the references therein.

Thus, to pick the optimal $n$-point set, we try to choose a point set, which has a very good equidistribution property that is lacking in random uniform sampling. One such point set in $[0,1]^s$ is called the $(t, m, s)$-net. To describe this point set we begin with a few definitions. Throughout these definitions $b\ge2$ is an integer
base, $s\ge1$ is an integer dimension and
$\mathbb{Z}_b=\{0,1,\dots,b-1\}$.

\begin{definition}
For $k_j\in\mathbb{N}_0$ and $c_j\in\mathbb{Z}_{b^{k_j}}$ for $j=1,\dots,s$, the set 
$$
\prod_{j=1}^s\Bigl[ \frac{c_j}{b^{k_j}},\frac{c_j+1}{b^{k_j}}\Bigr) 
$$
is a $b$-adic box of dimension $s$. 
\end{definition}

\begin{definition}
For integers $m\ge t\ge0$,
the points $\xb_1,\dots, $ $\xb_{b^m}  \in[0,1]^s$ 
are a $(t,m,s)$-net in base $b$
if every $b$-adic box of dimension $s$ with volume $b^{t-m}$
contains precisely $b^t$ of the $\xb_i$. 
\end{definition}

The nets have good equidistribution (low discrepancy) because
boxes $[0,a]$ can be efficiently approximated by unions of
$b$-adic boxes. Digital nets can attain a discrepancy of
$O((\log(n))^{s-1}/n)$. There is vast literature on easy construction of these point sets. For more details on digital nets we refer to \cite{dick:pill:2010, nied92}.

\subsubsection{Area preserving map to $\partial\cs$}
Now once we have a point set on $[0,1]^s$ we try to map it to $\partial\cs$ using a measure preserving transformation so that the equidistribution property remains. We describe the mapping in two steps. First we map the point set from $[0,1]^s$ to the hyper-sphere $\mathbb{S}^s = \{ \xb \in \real^{s+1} : \xb^T\xb = 1\}$. Then we map it to $\partial\cs$. The mapping from $[0,1]^s$ to $\mathbb{S}^s$ is based on \cite{brauchart2012quasi}.  

The cylindrical coordinates of the $d$-sphere, can be written as
\begin{align*}
\xb = \xb_s &= ( \sqrt{1 - t_s^s} \xb_{s-1}, t_s)\\ 
& \ldots \\
 \xb_2 &= (\sqrt{1 - t_2^2} \xb_1, t_2)\\
\xb_1 &= (\cos \phi, \sin \phi)
\end{align*}
where $0 \leq \phi \leq 2\pi, -1 \leq t_d \leq 1, \xb_d \in \mathbb{S}^d$ and $d = (1,\ldots, s)$. Thus, an arbitrary point $\xb \in \mathbb{S}^s$ can be represented through angle $\phi$ and heights $t_2, \ldots, t_s$ as,
\begin{align*}
\xb = \xb(\phi, t_2, \ldots, t_s), \qquad 0 \leq \phi \leq 2\pi, -1 \leq t_2, \ldots, t_s \leq 1.
\end{align*}
We map a point $\yb = (y_1, \ldots, y_s) \in [0,1)^s$ to $\xb \in \mathbb{S}^s$ using
\begin{align*}
\varphi_1(y_1) = 2\pi y_1, \qquad \varphi_d(y_d) = 1 - 2y_d \;\;\;(d = 2, \ldots, s)
\end{align*}
and cylindrical coordinates
\begin{align*}
\xb = \Phi_s(\yb) = \xb( \varphi_1(y_1), \varphi_2(y_2), \ldots, \varphi_s(y_s)).
\end{align*}

\begin{lemma}
\label{lem:pres_1}
Under the above notation, $\Phi_s : [0,1)^s \rightarrow \mathbb{S}^s$ is an area preserving map.
\end{lemma}

\begin{proof}
See \cite{brauchart2012quasi}.
\end{proof}

\begin{remark}
Instead of using $(t,m,s)$-nets and mapping to $\mathbb{S}^s$, we could have also used spherical $t$-designs, the existence of which was proved in \cite{bondarenko2013optimal}. However, construction of such sets is still a hard problem in large dimensions. We refer to \cite{brauchart2015distributing} for more details. 
\end{remark}

We consider the map from $\psi : \mathbb{S}^{s-1} \rightarrow\partial\cs$ defined as follows.
\begin{align}
\label{eq:psi}
\psi(\xb) =\sqrt{\tilde{b}}B^{-1/2}\xb + b.
\end{align}
The next result shows that this also an area-preserving map, in the sense of normalized surface measures.

\begin{lemma}
\label{lem:pres_2}
Let $\psi$ be a mapping from $\mathbb{S}^{s-1} \rightarrow\partial\cs$ as defined in \eqref{eq:psi}. Then for any set $A \subseteq \partial\cs$,
\begin{align*}
\sigma_s(A) = \lambda_s(\psi^{-1}(A))
\end{align*}
where, $\sigma_s, \lambda_s$ are the normalized surface measure of $\partial\cs$ and $\mathbb{S}^{s-1}$ respectively
\end{lemma}

\begin{proof}
Pick any $A \subseteq \partial\cs$. Then we can write,
\begin{align*}
\psi^{-1}(A) = \left\{ \frac{1}{\sqrt{\tilde{b}}}B^{1/2}(\xb - b) : \xb \in A \right\}.
\end{align*}
Now since the linear shift does not change the surface area, we  have,
\begin{align*}
\lambda_s(\psi^{-1}(A)) &= \lambda_s\left(\left\{ \frac{1}{\sqrt{\tilde{b}}}B^{1/2}(\xb - b) : \xb \in A \right\}\right)  \\
&= \lambda_s\left(\left\{ \frac{1}{\sqrt{\tilde{b}}}B^{1/2}\xb : \xb \in A \right\}\right) = \sigma_s(A),
\end{align*}
where the last equality follows from the definition of normalized surface measures. This completes the proof.
\end{proof}

Following Lemmas \ref{lem:pres_1} and \ref{lem:pres_2} we see that the map $$\psi \circ \Phi_{s -1} : [0,1)^s \rightarrow \partial\cs,$$ is a measure preserving map. Using this map  and the $(t,m,s-1)$ net in base $b$, we derive the optimal $b^m$-point set on $\partial\cs$. The
procedure is detailed as Algorithm \ref{algo:simulate}. 

\begin{algorithm}
\caption{Point Simulation on $\partial\cs$}\label{algo:simulate}
\begin{algorithmic}[1]
\State Input : $B, b, \tilde{b}$ to specify $\cs$ and $N = k^m$ points 
\State Output : $\xb_1, \ldots, \xb_N \in \partial\cs$
\State Generate $\yb_1, \ldots, \yb_N$ as a $(t,m,s-1)$-net in base $k$.
\For {$i \in 1, \ldots, N$}
\State $\xb_i = \psi \circ \Phi_{s-1}(\yb_i)$
\EndFor
\State \Return {$\xb_1, \ldots, \xb_N$}
\end{algorithmic}
\end{algorithm}

Figure \ref{fig:net2sph2ell} shows how we transform a $(0,7,2)$-net in
base 2 to a sphere and then to an ellipsoid. For more general
geometric constructions we refer to \cite{basu:owen:2015, basu2015scrambled}.

\begin{figure*}
\centering
\includegraphics[width=\linewidth]{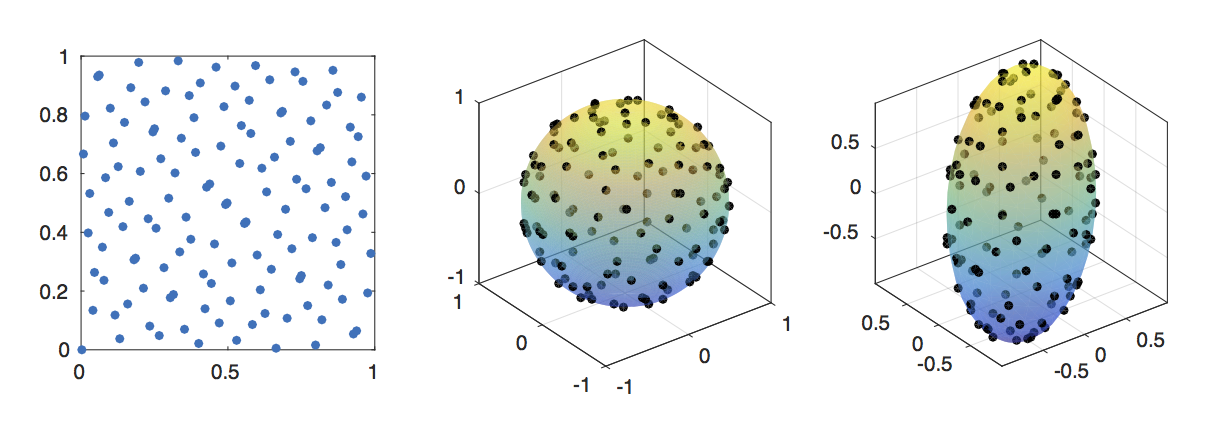}
\caption{\label{fig:net2sph2ell} The left panel shows a $(0, 7,
  2)$-net in base $2$ which is mapped to a sphere in 3 dimensions
  (middle panel) and then mapped to the ellipsoid as seen in the right
  panel.}
\end{figure*}

\subsection{Convergence guarantees}
Consider the problem $\cp(N)$ as stated in \eqref{eq:QP_via_lin}. We
shall show that asymptotically as $N \rightarrow \infty$, we get back
the original problem $\cp$ as stated in \eqref{eq:QCQP}. We shall also
prove some finite sample results to give some error bounds on the
solution to $\cp(N)$. We begin with a few notation. Let $\xb^*,
\xb^*(N)$ denote the solution to $\cp$ and $\cp(N)$ respectively. Let
$f(\cdot)$ denote the objective function.

\begin{theorem}
\label{thm:converge}
Let $\cp$ and $\cp(N)$ be the optimization problems defined in
\eqref{eq:QCQP} and \eqref{eq:QP_via_lin} respectively. Then, $\cp(N)
\rightarrow \cp$ as $N \rightarrow \infty$ in the sense that
$$\lim_{N \rightarrow \infty} \norm{\xb^*(N) - \xb^*} = 0.$$ 
\end{theorem}

\begin{proof}
Fix any $N$. Let $\ct_{N}$ denote the optimal bounded cover constructed with $N$ points on $\partial\cs$. Note that to prove the result, it is enough to show that $\ct_{N} \rightarrow \cs$ as $N \rightarrow \infty$. This guarantees that linear constraints of $\cp(N)$ converge to the quadratic constraint of $\cp$, and hence the two problems match. Now since $\cs \subseteq \ct_N$ for all $N$, it is easy to see that $\cs \subseteq \lim_{N \rightarrow \infty} \ct_N.$ 

To prove the converse, let $t_0 \in \lim_{N \rightarrow \infty} \ct_N$ but $t_0 \not \in \cs$. Thus, $d(t_0, \cs) > 0$. Let $t_1$ denote the projection of $t_0$ onto $\cs$. Thus, $t_0 \neq t_1 \in \partial\cs$. Choose $\epsilon$ to be arbitrarily small and consider any region $A_\epsilon$ on $\partial\cs$ with diameter less than $\epsilon$. Since, we are working with a limiting case, there exists infinitely many points in $A_\epsilon$. Thus there exists a point $t^* \in A_\epsilon$, the tangent plane through which cuts the line joining $t_0$ and $t_1$. Thus, $t_0 \not\in \lim_{N \rightarrow \infty} \ct_N$. Hence, we get a contradiction and the result is proved.
\end{proof}

Note that Theorem \ref{thm:converge} shows that $\lim_{N \rightarrow
  \infty} \norm{\xb^*(N) - \xb^*} = 0$ and hence, $\lim_{N \rightarrow
  \infty} \left| f(\xb^*(N)) - f(\xb^*)\right| = 0.$ Now we prove
state some finite sample results.

\begin{theorem}
\label{thm:upp_bound}
Let $g : \mathbb{N} \rightarrow \real$ such that $\lim_{n \rightarrow
  \infty} g(n) = 0.$ Further assume that $\norm{\xb^*(N) - \xb^*} \leq
C_1 g(N)$ for some constant $C_1 > 0$. Then,
\begin{align*}
\left| f(\xb^*(N)) - f(\xb^*) \right| \leq C_2 g(N)
\end{align*}
where $C_2 > 0$ is a constant. 
\end{theorem}

\begin{proof}
Note that $f(\xb) = (\xb - \bsa)^TA(\xb - \bsa)$ and
$\nabla f(\xb) = 2A(\xb - \bsa)$. Now, note that we can write
\begin{align*}
f(\xb) &= f(\xb^*) + \int_{0}^1 \langle \nabla f(\xb^* + t(\xb - \xb^*)), \xb -\xb^* \rangle dt \\
&= f(\xb^*) +  \langle \nabla f(\xb^*), \xb - \xb^* \rangle \\
&\;\;\; + \int_{0}^1 \langle \nabla f(\xb^* + t(\xb - \xb^*)) - \nabla f(\xb^*), \xb -\xb^* \rangle dt\\
&= I_1 + I_2 + I_3 \text{ (say) }.
\end{align*}
Now, we can bound the last term as follows. Observe that using Cauchy-Schwarz inequality,
\begin{align*}
\left|I_3\right| &\leq \int_{0}^1 \left|\langle \nabla f(\xb^* + t(\xb - \xb^*)) - \nabla f(\xb^*), \xb -\xb^* \rangle\right| dt\\
&\leq \int_{0}^1 \norm{ \nabla f(\xb^* + t(\xb - \xb^*)) - \nabla f(\xb^*)}\norm{ \xb -\xb^* } dt\\
&\leq 2  \sigma_{\max}(A) \int_{0}^1 \norm{t(\xb -\xb^*)}\norm{\xb -\xb^*} dt \\
&=  \sigma_{\max}(A)\norm{\xb -\xb^*}^2,
\end{align*}
where $\sigma_{\max}(A)$ denotes the highest singular value of $A$.
Thus, we have
\begin{align}
\label{eq:step1}
f(\xb) &= f(\xb^*) + \langle \nabla f(\xb^*), \xb - \xb^* \rangle + \tilde{C} \norm{\xb -\xb^*}^2
\end{align}
where $|\tilde{C}| \leq \sigma_{\max}(A)$. Furthermore,
\begin{align}
\label{eq:step2}
|\langle \nabla f(\xb^*), &\xb^*(N) - \xb^* \rangle| \nonumber\\
&= \left|\langle 2A(\xb^* - a), \xb^*(N) - \xb^* \rangle\right| \nonumber\\
&\leq 2 \sigma_{\max}(A) (\norm{\xb^*} + \norm{\bsa}) \norm{\xb^*(N) - \xb^*} \nonumber \\
&\leq 2 C_1 \sigma_{\max}(A) (s +  \norm{\bsa}) g(N)
\end{align}
Combining \eqref{eq:step1} and \eqref{eq:step2} we have, 
\begin{align*}
\left| f(\xb^*(N)) - f(\xb^*) \right| \leq C_2 g(N)
\end{align*}
for some constant $C_2 > 0$. Thus, the result follows. 
\end{proof}

Note that the function $g$ gives us an idea about how fast $\xb^*(N)$
converges $\xb^*$. To help, identify the function $g$ we state the
following results.

\begin{lemma}
\label{lem:step1}
If $f(\xb^*) = f(\xb^*(N))$, then $\xb^* = \xb^*(N)$. Furthermore, if
$f(\xb^*) \geq f(\xb^*(N))$, then $\xb^* \in \partial\cu$ and
$\xb^*(N) \not\in \cu$, where $\cu = \cs \cap \{\xb : C\xb \leq c \}$
is the feasible set for \eqref{eq:QCQP}
\end{lemma}

\begin{proof}
Let $\cv = \ct_N \cap \{\xb : C\xb \leq c \}$. It is easy to see that $\cu \subseteq \cv$. Assume $f(\xb^*) = f(\xb^*(N))$, but $\xb^* \neq \xb^*(N)$. Note that $\xb^* , \xb^*(N) \in \cv$. Since $\cv$ is convex, consider a line joining $\xb^*$ and $\xb^*(N)$. For any point $\lambda_t = t \xb^* + (1 - t) \xb^*(N)$,
\begin{align*}
f(\lambda_t) \leq t f(\xb^*) + (1-t) f(\xb^*(N)) = f(\xb^*(N)). 
\end{align*}
Thus, $f$ is constant on the line joining $\xb^*$ and $\xb^*(N)$. But, it is known that $f$ is strongly convex since $A$ is positive definite. Thus, there exists only one unique minimum. Thus, we have a contradiction, which proves $\xb^* = \xb^*(N)$

Now let us assume that $f(\xb^*) \geq f(\xb^*(N))$. Clearly, $\xb^*(N) \not\in \cu$. Suppose $\xb^* \in \accentset{\circ}{\cu}$. Let $\tilde{\xb} \in \partial\cu$ denote the point on the line joining $\xb^*$ and $\xb^*(N)$.  Clearly, $\tilde{\xb} = t\xb^* + (1-t)\xb^*(N)$ for some $t > 0$. Thus,
\begin{align*}
f(\tilde{\xb}) < t f(\xb^*) + (1-t) f(\xb^*(N)) \leq f(\xb^*)
\end{align*}
But $\xb^*$ is the minimizer over $\cu$. Thus, we have a contradiction, which gives $\xb^* \in \partial\cu$. This completes the proof.
\end{proof}

\begin{lemma}
\label{lem:step2}
Following the notation of Lemma \ref{lem:step1}, if $\xb^*(N) \not\in
\cu$, then $\xb^*$ lies on $\partial\cu$ within the conic cap of $\cu$
generated from $\xb^*(N)$.
\end{lemma}

\begin{proof}
Since the gradient of $f$ is linear, the result is easy to see from the proof of the second assertion in Lemma \ref{lem:step1}.
\end{proof}

Now we can identify the function $g$ by considering the maximum
distance of the points lying on the intersection of the cone and
$\cs$. This is highly dependent on the shape of $\cs$ and on the cover
$\ct_N$. Explicit calculation can give us explicit rates of
convergence, which we leave as future work.

%% file: Expt_results.tex
A detailed study of solving the original optimization
problem \eqref{eq:orig_qp} via the efficient solution approach of
Section \ref{sec:soln_prob} has been done in \cite{basu_user}. For
compactness, we only report the results regarding the experimentation
with modeling interactions.

\subsection{Need for modeling interaction}
We first show that if we ignore the interaction effect, we will get an
increasingly worse solution as the dimension of the problem
increases. We consider the following simple optimization problem for
our purposes.
\begin{equation}
\label{simple_QCQP}
\begin{aligned}
& \underset{\xb}{\text{Minimize}} & &  -\xb^T\pb + \frac{\gamma}{2} \xb^T\xb\\
& \text{subject to}&  &\xb'\rb  \leq P\\
& & & 0 \leq \xb \leq 1
\end{aligned}
\end{equation} 
where the truth is $\pb = -\Qb_p \xb$ and $\rb = \Qb_r \xb$ for
positive definite matrices $\Qb_p, \Qb_r$ as defined in
Section \ref{sec:interact}. We solve the true optimization problem,
\begin{equation}
\label{truth}
\begin{aligned}
& \underset{\xb}{\text{Minimize}} & &  \xb^T\left( \Qb_p + \frac{\gamma}{2}\Ib\right) \xb\\
& \text{subject to}&  &\xb'\Qb_r\xb  \leq P\\
& & & 0 \leq \xb \leq 1
\end{aligned}
\end{equation} 
to get $\bsx^*$. Now, ignoring the dependency structure we use $\pb =
diag(\Qb_p)$ and $\rb = diag(\Qb_r)$ to solve the
problem \eqref{simple_QCQP} to get $\hat{\xb}$. Note that the
diagonals are just the single slot estimates.

We simulate two random instances of the single slot estimate and
create $diag(\Qb_p), diag(\Qb_r)$ as a decreasing function of the slot
position. We then simulate the rest of the matrix $\Qb_p$ and $\Qb_r$
according to the structure in Section \ref{sec:interact}. For
different values of sample size $n$, we calculate the relative error
as
\begin{align*}
err(n) = \frac{\hat{\xb}^T\left( \Qb_p
+ \frac{\gamma}{2}\Ib\right) \hat{\xb} - (\xb^*)^T\left( \Qb_p
+ \frac{\gamma}{2}\Ib\right) \xb^* }{(\xb^*)^T\left( \Qb_p
+ \frac{\gamma}{2}\Ib\right) \xb^*}
\end{align*}
Figure \ref{fig:obj_val} shows the log of two different objective values. It
can be seen that we start getting increasing worse solutions as the
dimension increases. The fluctuations are because of the randomness in
the simulations.

\begin{figure}[!h]
\centering
\includegraphics[width = \linewidth]{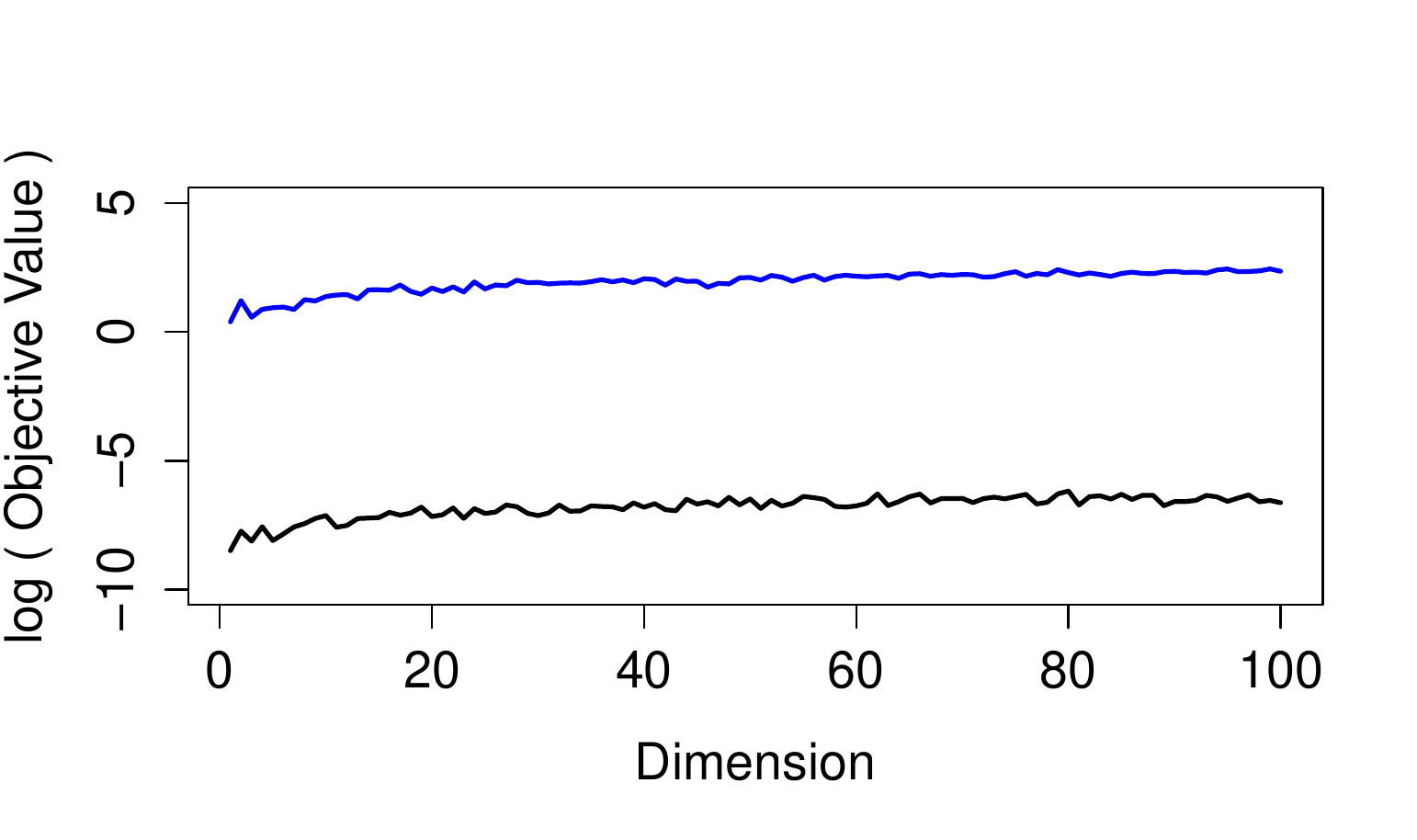}
\caption{\label{fig:obj_val} 
The black line shows the true objective value obtained from solving
problem \eqref{truth}. The blue line shows the objective value using
the solution from the ignored dependency problem.}
\end{figure}

Figure \ref{fig:err} show the log of the relative error as a function
of the dimension. It can be seen from our simulation, that we get an
average relative error of about $6 \times 10^8$ which is extremely
high of a cost to pay for ignoring the interaction effect.

\begin{figure}[!h]
\centering
\includegraphics[scale = 0.55]{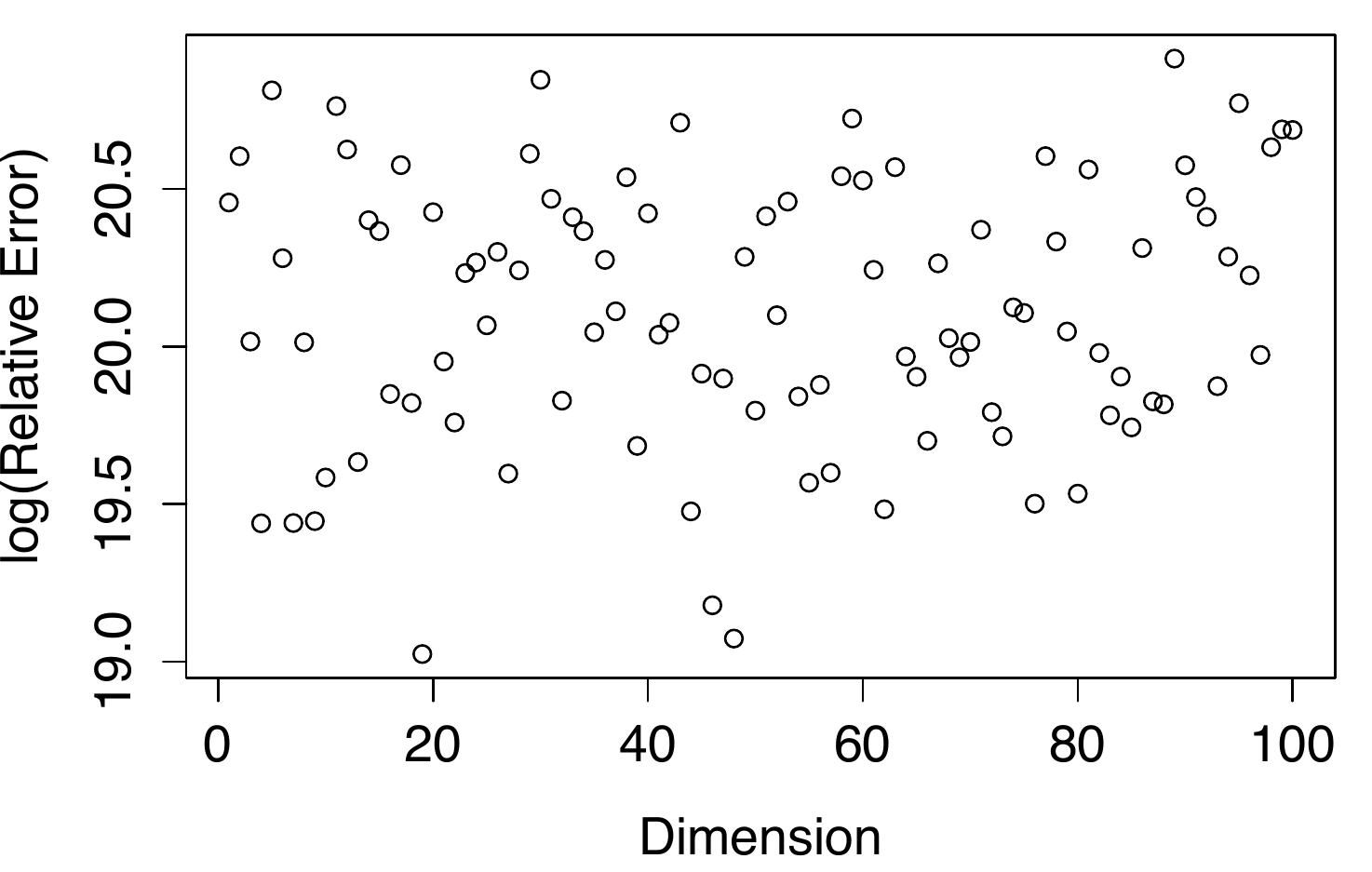}
\caption{\label{fig:err} 
Plotting the log(relative error) against the dimension of the problem}
\end{figure}

\begin{table*}[!t]
\caption{\label{tab:obj_time}The optimal objective value and convergence time}
\begin{center}
 \begin{tabular}{|| c | c | c | c | c |c |c ||} \hline $n$ & Our
 method & Sampling on $[0,1]^n$ & Sampling on $\mathbb{S}^n$ & SDP
 relaxation & RLT relaxation & Exact \\[0.5ex]
\hline\hline 
\multirow{2}{*} {5} & 3.00 & 2.99 & 2.95  & 3.07  & 3.08 & 3.07 \\ 
& (4.61s) & (4.74s) & (6. 11s) & (0.52s) & (0.51s) & (0.49) \\
 \hline
\multirow{2}{*} {10} & 206.85 & 205.21 & 206.5  & 252.88  & 252.88 & 252.88 \\ 
& (5.04s) & (5.65s) & (5.26s) & (0.53s) & (0.51s) & (0.51) \\
\hline
\multirow{2}{*} {20} & 6291.4 & 4507.8 & 5052.2  & 6841.6  & 6841.6 & 6841.6 \\ 
& (6.56s) & (6.28s) & (6.69s) & (2.05s) & (1.86s) & (0.54) \\
\hline
\multirow{2}{*} {50} & 99668 & 15122 & 26239 & $1.11\times10^5$  & $1.08\times10^5$ & $1.11\times10^5$ \\ 
& (15.55s) & (18.98s) & (17.32s) & (4.31s) & (2.96s) & (0.64) \\
\hline
\multirow{2}{*} {100} & $1.40\times10^6$ & 69746 & $1.24\times10^6$  & $1.62\times10^6$  & $1.52\times10^6$& $1.62\times10^6$ \\ 
& (58.41s) & (1.03m) & (54.69s) & (30.41s) & (15.36s) & (2.30s) \\
 \hline
\multirow{2}{*} {1000} & $2.24\times10^7$ & $8.34 \times 10^6$ & $9.02\times10^6$  & \multirow{2}{*} {NA}  & \multirow{2}{*} {NA} & \multirow{2}{*} {NA} \\ 
& (14.87m) & (15.63m) & (15.32m) &  &  & \\
\hline
\multirow{2}{*} {$10^5$} & $3.10\times10^8$ & $7.12 \times 10^7$ & $ 8.39\times10^7$  & \multirow{2}{*} {NA}  & \multirow{2}{*} {NA} & \multirow{2}{*} {NA} \\ 
& (25.82m) & (24.59m) & (27.23m) &  &  & \\
 \hline
\multirow{2}{*} {$10^6$} & $3.91\times10^9$ & $2.69 \times 10^8$ & $7.53\times10^8$  & \multirow{2}{*} {NA}  & \multirow{2}{*} {NA} & \multirow{2}{*} {NA} \\ 
& (38.30m) & (39.15m) & (37.21m) &  &  & \\
 \hline
 \end{tabular}
\end{center}
\end{table*}

\subsection{Comparative Study of QCQP Solutions}
We compare our proposed technique to the current state-of-the-art
solvers of QCQP. Specifically we compare it to the SDP and RLT
relaxation procedures. For small enough problems, we also compare our
method to the exact solution by interior point methods. Furthermore,
we provide empirical evidence to show that our sampling technique is
better than other simpler sampling procedures such as uniform sampling
on the unit square or on the unit sphere and then mapping it
subsequently to our domain as in Algorithm \ref{algo:simulate}. We
begin by considering a very simple QCQP for the form
\begin{equation}
\label{test_qcqp}
\begin{aligned}
& \underset{\xb}{\text{Minimize}} & &  \xb^T \Ab \xb\\
& \text{subject to}&  &(\xb - \xb_0)^T \Bb (\xb - \xb_0)  \leq \tilde{b},\\
& & & \lb \leq \xb \leq \ub
\end{aligned}
\end{equation}
We randomly sample $\Ab, \Bb, \xb_0$ and $\tilde{b}$. The lower bound, $\lb$ and upper bounds $\ub$ are chosen in a way such that they intersect the ellipsoid. We vary the dimension $n$ of the problem and tabulate the final objective value as
well as the time taken for the different procedures to converge in
Table \ref{tab:obj_time}. Throughout our simulations we have chosen $\eta = 2$ and the number of optimal points as $N = \max (1024, 2^m) $, where $m$ is the smallest integer such that $2^m \geq 10n$.

Note that even though the SDP and the interior point methods converge very efficiently for small values of $n$, they cannot scale to values of $n \geq 1000$, which is where the strength of our method becomes evident. From Table \ref{tab:obj_time} we observe that the relaxation procedures SDP and RLT fail to converge within an hour of computation time for $n \geq 1000$, whereas all the approximation procedures can easily scale up to $n = 10^6$ variables.  We can further notice that SDP performs slightly better than RLT for higher dimensions. 

Furthermore, our procedure gives the best approximation result when compared to the remaining two sampling
schemes. Lemma \ref{lem:step1} shows that if the both the objective
values are the same then we indeed get the exact solution. To see how much the approximation deviates from the truth, we also tabulate the relative error, i.e. $\norm{\xb^*(N) - \xb^*}/\norm{\xb^*}$ for each
of the sampling procedures in Table \ref{tab:err}. We omit SDP and RLT results in Table \ref{tab:err} since both of them produce a solution very close to the exact minimum for small $n$. From the results in Table \ref{tab:err} it is clear that our procedure gets the smallest relative error compared to the other sampling schemes that we tried.

\begin{table}[!t]
\caption{\label{tab:err}The Relative Error : $\norm{\xb^*(N) - \xb^*}/\norm{\xb^*}$}
\begin{center}
 \begin{tabular}{|| c | c | c | c ||} 
 \hline
 $n$ & Our method & Sampling on $[0,1]^n$ & Sampling on $\mathbb{S}^n$  \\[0.5ex]
\hline\hline 
5 &0.0615 	& 0.0828 & 0.0897\\
\hline
10&0.0714 	&	 0.1530&0.1229\\
\hline
20 &0.0895&	0.2455&	0.2368\\
\hline
50&0.3352&	3.8189&	1.0472\\
\hline
100&0.8768	&13.3709	&2.0849\\
 \hline
 \end{tabular}
\end{center}
\end{table}

%% file: Discussion.tex
In this paper, we look at the problem of trading off multiple
objectives while ranking recommendations over mutliple slots. We give
a deterministic serving plan under general constraints for a single
slot and then give a formulation for the multi-slot setting assuming
dependence in interaction of the items in the different slots. We
characterize the conditions under which it is possible to efficiently
solve the problem and give an approximation algorithm for the QCQP,
which involves relaxing the constraints via a non-trivial sampling
scheme. This method can scale up to very large problem sizes while
generating solutions which have good theoretical properties of
convergence.